\documentclass[11pt]{article}

\usepackage{jmlr2e}

\usepackage{microtype}
\usepackage{graphicx}
\usepackage{subfigure}
\usepackage{booktabs} 
\usepackage{amssymb,amsmath,amsfonts}

\usepackage{enumerate}
\usepackage[shortlabels]{enumitem}
\usepackage{natbib}

\newcommand{\Div}{\textsc{D} }
\newcommand{\Xent}{\textsc{H}}

\newcommand{\cA}{\mathcal{A}}
\newcommand{\cB}{\mathcal{B}}
\newcommand{\cX}{\mathcal{X}}

\DeclareMathOperator*{\argmin}{arg\,min}

\DeclareMathOperator{\Prtxt}{Pr}

\newcommand{\ifn}{\mathbf{1}} 
\newcommand{\evp}[2]{\mathbb{E}_{#2} \left[#1\right]} 
\newcommand{\abs}[1]{\left| #1 \right|}

\newcommand{\prp}[2]{\Prtxt_{#2} \left(#1\right)}

\newcommand{\lrp}[1]{\left(#1\right)}

\firstpageno{1}

\begin{document}

\title{Sharp finite-sample concentration of independent variables}

\author{\name Akshay Balsubramani \email akshay7@gmail.com \\
       }


\maketitle

\begin{abstract}
We show an extension of Sanov's theorem on large deviations, controlling the tail probabilities of i.i.d. random variables with matching concentration and anti-concentration bounds. 
This result has a general scope, applies to samples of any size, and has a short information-theoretic proof using elementary techniques. 
\end{abstract}

\vspace{2em}
Independently and identically distributed (i.i.d.) data  are drawn from a distribution $P$. 
A central focus in statistics, machine learning, and probability is what can be gleaned about $P$ from a sample of these data -- how much data must be sampled for the empirical distribution of the data to concentrate near $P$?

\section{Setup}

We describe distributions $P$ and $Q$ over a measurable space $\cX$ using the quantities of information theory. 
The entropy of a distribution $P$ is $\Xent (P) := \evp{\ln \frac{1}{P(x)}}{x \sim P}$. 
The relative entropy of $Q$ with respect to $P$ is $\Div (Q \mid\mid P) := \evp{\ln \frac{Q(x)}{P(x)}}{x \sim Q}$. 
The cross entropy of $P$ with respect to $Q$ is $\Xent (Q , P) := \evp{\ln \frac{1}{P(x)}}{x \sim Q}$. 
The empirical measure of any sample $Z = (z_1, \dots, z_n) \in \cX^{n}$ is $\hat{\mu}_{Z}$, a distribution over $\cX$. \footnote{It assigns probability to any $E \subseteq \cX$ of $\hat{\mu}_Z (E) := \frac{1}{n} \sum_{i=1}^{n} \ifn ( z_i \in E )$.} 
In our case, $Z \sim P^n$ is often an i.i.d. data sample -- a random variable; as shorthand, the resulting (random) empirical measure over $\cX$ is written $\hat{P}_n$.

Concentration of the empirical distribution $\hat{P}_n$ to $P$ is addressed by Sanov's theorem \citep{Sanov57}, when we are looking to measure if $\hat{P}_n$ falls in a set $\cA$ of distributions over $\cX$. 
Sanov's theorem bounds the probability of any such $\cA$ if $\cA$ is ``nice" (convex), in the limit of infinite samples. 

\begin{theorem}[Sanov's Theorem \citep{CT06}]
\label{thm:asymp-sanov}
Fix a distribution $P$ over $\cX$ and a convex set $\cA$ of such distributions. Then
$$
\lim_{n \to \infty} \; \frac{1}{n} \ln \prp{ \hat{P}_n \in \cA}{} 
= - \min_{Q \in \cA} \Div (Q \mid\mid P)
= - \Div (P_{\cA}^{*} \mid\mid P)
$$
where $P_{\cA}^{*} \in \argmin_{Q \in \cA} \Div (Q \mid\mid P)$ is called the I-projection of $P$ on $\cA$.
\end{theorem}

The right-hand side is typically a finite negative quantity, so this precisely describes the rate of exponential decay of ``tail" events which decay exponentially fast in $n$. 
For finite $n$, this is not directly useful, but a variety of concentration bounds (e.g. Chernoff-type bounds) can be derived by approximating the decay for extreme-enough tail events and finite $n$. 
These typically apply to special subcases of the setting of Theorem \ref{thm:asymp-sanov} only, and can be loose.

\paragraph*{Linear families. }
For any real-valued functions $f_1, f_2, ..., f_k$ on $\cX$ and $\alpha_1, \alpha_2, \dots, \alpha_k \in \mathbb{R}$, the set
$
\{
P : \evp{f_i (a)}{a \sim P} = \alpha_i \; , \; \forall i \in [k]
\} 
$
is called a linear family of distributions.
For any $Q \in \cA$ if $\cA$ is convex, 
\begin{align}
\label{eq:pythagsubset}
\Div (Q \mid\mid P ) \geq \Div (Q \mid\mid P_{\cA}^{*} ) + \Div (P_{\cA}^{*} \mid\mid P )
\end{align}
with equality if $\cA$ is a linear family \citep{csiszar2004information}. 
This is the Pythagorean (in)equality for I-projections.

\section{Main results}

\subsection{A tight characterization with the I-projection}

We prove a tight general-purpose information-theoretic characterization of the probability of $\cA$ under the empirical measure $\hat{P}_{n}$. 
This is stated in terms of the conditional distribution $\displaystyle \mu_{\cA} (y) := \prp{Z = y \mid \hat{\mu}_{Z} \in \cA}{Z \sim P^{n}}$, for $y \in \cX^n$. 

\begin{theorem}
\label{thm:sanovsubset}
Fix a distribution $P$ over $\cX$ and a convex set $\cA$ of such distributions. 
For any $\cB \subseteq \cA$, 
\begin{align}
\label{eq:sanovsubsetexact}
\frac{1}{n} \ln \prp{ \hat{P}_n \in \cB}{} \leq - \Div (P_{\cA}^{*} \mid\mid P ) - \frac{1}{n} \Div ( \mu_{\cB} \mid\mid P_{\cA}^{*n} )
\end{align}
If $\cA$ is a linear family, equality holds. 
\end{theorem}

It is instructive to consider $\cB = \cA$. 
The upper bound of Theorem \ref{thm:sanovsubset}  
resembles the asymptotic Sanov bound $- \Div (P_{\cA}^{*} \mid\mid P)$ of Theorem \ref{thm:asymp-sanov}. 
However, it is strengthened by a term $\frac{1}{n} \Div ( \mu_{\cA} \mid\mid P_{\cA}^{*n} )$, which vanishes as $n \to \infty$, by Sanov's theorem.

\subsection{A more general extension}

Instead of using I-projections, we can state a slightly more general version of Theorem \ref{thm:sanovsubset}. 
Define $\mu_{\cA}$'s marginal $\omega_{\cA}(u) := \prp{Z_i = u \mid \hat{\mu}_{Z} \in \cA}{Z \sim P^{n}}$ for $u \in \cX$, along any coordinate $i = 1, \dots, n$.\footnote{These are all the same, by exchangeability.}

\begin{theorem}
\label{thm:sanov}
Fix any distribution $P$ over $\cX$ and a set $\cA$ of such distributions. Then: 
\begin{align}
\label{eq:sanovexact}
\frac{1}{n} \ln \prp{ \hat{P}_n \in \cA}{} = - \Div (\omega_{\cA} \mid\mid P ) - \frac{1}{n} \Div ( \mu_{\cA} \mid\mid \omega_{\cA}^n )
\end{align}
Also, 
\begin{enumerate}[1)]
    \item
    $
    \frac{1}{n} \ln \prp{ \hat{P}_{n} \in \cA}{} 
    \stackrel{(a)}{\leq}
    - \Div (\omega_{\cA} \mid\mid P ) 
    \stackrel{(b)}{\leq} 
    - \Div (P_{\cA}^{*} \mid\mid P)
    $
    \vspace{-0.5em}
    \item
    $\displaystyle
    \frac{1}{n} \ln \prp{ \hat{P}_{n} \in \cA}{}
    \stackrel{(a)}{\geq}
    - \Xent (\omega_{\cA} , P ) 
    \stackrel{(b)}{\geq} 
    - \max_{Q \in \cA} \Xent (Q , P )
    $
\end{enumerate}
The inequalities $(1b)$ and $(2b)$ require $\cA$ to be a convex set.
\end{theorem}

The upper bound 
$- \Div (\omega_{\cA} \mid\mid P ) = - \Xent (\omega_{\cA} , P ) + \Xent (\omega_{\cA})$ 
and lower bound $- \Xent (\omega_{\cA} , P ) $ 
(part (2) of the theorem) can be compared:
\begin{align}
\abs{ \frac{\mbox{upper bound}}{\mbox{lower bound}} - 1 }
&= \abs{ \frac{- \Xent (\omega_{\cA} , P ) + \Xent (\omega_{\cA})}{- \Xent (\omega_{\cA} , P )} - 1 }
= \frac{\Xent (\omega_{\cA})}{\Xent (\omega_{\cA} , P )}
\end{align}
In situations where $\Xent (\omega_{\cA}) \ll \Xent (\omega_{\cA} , P )$ ($\cA$ represents a very atypical event under $P$, so that conditioning on $\cA$ significantly changes $P$), 
the bounds tend to match well.

For finite $n$, the slack in the upper bound is $- \frac{1}{n} \Div ( \mu_{\cA} \mid\mid \omega_{\cA}^n )$. 
This quantity $\Div ( \mu_{\cA} \mid\mid \omega_{\cA}^n )$ is called the \emph{total correlation} between the $n$ marginal distributions of $\mu_{\cA}$. 
The total correlation, a multivariate generalization of the mutual information \citep{watanabe1960information}, is the information gained by knowing the joint distribution instead of knowing just the marginals. 
It decreases as $n$ increases here because the marginals are nearly independent, to become negligible as $n \to \infty$ in Sanov's theorem.


\section{Discussion}

The proof techniques here are from Csisz{\'a}r \citep{csiszar1984sanov}, underpinned by an information-theoretic identity (\cite{csiszar1984sanov}, Eq. 2.11). 
This expresses large deviation probabilities of an $n$-sample random variable in terms of the information geometry of its conditional distribution. 

\begin{lemma}
\label{lem:eqsanov}
Fix a distribution $P$ over $\cX$ and a set $\cA$ of distributions over $\cX$. 
Then 
\begin{align*}
- \ln \prp{\hat{P}_n \in \cA}{} 
&= \Div ( \mu_{\cA} \mid\mid P^{n} ) 
= n \Div (\omega_{\cA} \mid\mid P ) + \Div ( \mu_{\cA} \mid\mid \omega_{\cA}^n )
\end{align*}
\end{lemma}


Sanov's theorem has notably been bounded in a non-asymptotic form in different situations, including for finite discrete spaces $\cX$ \citep{CT06} and using vastly more advanced techniques (\cite{DZ93}, Ex. 6.2.19).
The distinctive additional features of this manuscript's result are the elementary proof and general scope. 
Extensions and applications are left to future work.


\section{Proofs}


\begin{proof}[Proof of Theorem \ref{thm:sanovsubset}]
To prove \eqref{eq:sanovsubsetexact}, 
\begin{align*}
\frac{1}{n} \ln \prp{\hat{P}_n \in \cB}{} 
&\stackrel{(a)}{=} - \Div (\omega_{\cB} \mid\mid P ) - \frac{1}{n} \Div ( \mu_{\cB} \mid\mid \omega_{\cB}^n ) \\
&\stackrel{(b)}{\leq} - \Div (P_{\cA}^{*} \mid\mid P ) - \Div (\omega_{\cB} \mid\mid P_{\cA}^{*} ) - \frac{1}{n} \Div ( \mu_{\cB} \mid\mid \omega_{\cB}^n ) \\
&\stackrel{(c)}{=} - \Div (P_{\cA}^{*} \mid\mid P ) - \frac{1}{n} \Div ( \mu_{\cB} \mid\mid P_{\cA}^{*n} )
\end{align*}
where (a) invokes Lemma \ref{lem:eqsanov} with the set $\cB$, (b) uses the Pythagorean inequality for I-projections \eqref{eq:pythagsubset} (since $\omega_{\cB} \in \cB \subseteq \cA$), and (c) uses Lemma \ref{lem:core-it-identity}. 
\end{proof}

\begin{proof}[Proof of Lemma \ref{lem:eqsanov}]
First, note that for all $y \in \cX^n$, 
by definition $\mu_{\cA} (y) = \frac{P^{n} (y)}{ \prp{ \hat{P}_{n} \in \cA}{}} \ifn (\hat{\mu}_{y} \in \cA)$. 
Therefore, 
\begin{align}
\Div ( \mu_{\cA} \mid\mid P^{n} ) 
&= \evp{ \ln \frac{\mu_{\cA} (Z)}{P^{n} (Z)} }{Z \sim \mu_{\cA}} 
= \frac{ \evp{ \ln \frac{\mu_{\cA} (Z)}{P^{n} (Z)} \ifn( \hat{\mu}_{Z} \in \cA) }{Z \sim P^{n}} }{\prp{ \hat{P}_{n} \in \cA}{}} \nonumber \\
&= \frac{ \evp{ - \ln \prp{ \hat{P}_{n} \in \cA}{} \ifn( \hat{\mu}_{Z} \in \cA) }{Z \sim P^{n}} }{\prp{ \hat{P}_{n} \in \cA}{}} 
= 
\label{eq:sanovbud1}
- \ln \prp{ \hat{P}_{n} \in \cA}{}
\end{align}
Calling Lemma \ref{lem:core-it-identity} (with $P, \mu_{\cA}, \omega_{\cA}$) shows that $\Div ( \mu_{\cA} \mid\mid P^{n} ) - \Div ( \mu_{\cA} \mid\mid \omega_{\cA}^n ) = n \Div ( \omega_{\cA} \mid\mid P )$. Combining this with \eqref{eq:sanovbud1} gives the result.
\end{proof}

\begin{lemma}
\label{lem:core-it-identity}
Fix a distribution $P$ over $\cX$ and a set $\cA$ of distributions over $\cX$, and define $\mu_{\cA}, \omega_{\cA}$ with respect to these. 
For any distribution $Q$ over $\cX$, 
\begin{align*}
\Div ( \mu_{\cA} \mid\mid Q^{n} ) - \Div ( \mu_{\cA} \mid\mid \omega_{\cA}^n ) = n \Div ( \omega_{\cA} \mid\mid Q )
\end{align*}
\end{lemma}
\begin{proof}[Proof of Lemma \ref{lem:core-it-identity}]
\begin{align}
\Div ( \mu_{\cA} \mid\mid Q^{n} ) &- \Div ( \mu_{\cA} \mid\mid \omega_{\cA}^n ) 
= \evp{ \ln \frac{\omega_{\cA}^n (Z)}{Q^{n} (Z)} }{Z \sim \mu_{\cA}} 
= \sum_{i=1}^{n} \evp{ \ln \frac{\omega_{\cA} (Z_i)}{Q (Z_i)} }{Z \sim \mu_{\cA}} \nonumber \\
&\stackrel{(a)}{=} \sum_{i=1}^{n} \evp{ \ln \frac{\omega_{\cA} (u)}{Q (u)} }{u \sim \omega_{\cA}} 
= \sum_{i=1}^{n} \Div ( \omega_{\cA} \mid\mid Q )
=
\label{eq:sanovbud2}
n \Div ( \omega_{\cA} \mid\mid Q )
\end{align}
where $(a)$ is because of the definition of the marginal distribution $\omega_{\cA}$. 
\end{proof}

\begin{proof}[Proof of Theorem \ref{thm:sanov}]
Lemma \ref{lem:eqsanov} is equivalent to Eq. \eqref{eq:sanovexact}. 

We get part (1a) by bounding Eq. \eqref{eq:sanovexact} using the expression $\Div ( \mu_{\cA} \mid\mid \omega_{\cA}^n ) \geq 0$. 
Using the convexity of $\cA$, and the fact that $\omega_{\cA} \in \cA$ (this can be proved\footnote{
For any set $F \subseteq \cX$, $\omega_{\cA}$ is a linear combination of distributions in $\cA$: 
\begin{align*}
\omega_{\cA} (F) &= \int_{Z \in \cX^n} \mu_{\cA} (Z) \ifn( Z_i \in F) 
= \int_{Z \in \cX^n} \mu_{\cA} (Z) \lrp{ \frac{1}{n} \sum_{i=1}^{n} \ifn( Z_i \in F) } \\
&= \int_{Z \in \cX^n} \mu_{\cA} (Z) \hat{\mu}_Z (F) 
= \int_{\hat{\mu}_Z \in \cA } \mu_{\cA} (Z) \hat{\mu}_Z (F)
\end{align*}
where the last equality is because $\{ Z : \hat{\mu}_Z \notin \cA \}$ are given measure zero by $\mu_{\cA}$. Since $\cA$ is convex, $\omega_{\cA} \in \cA$.
} from the definitions), 
and further applying $\Div (\omega_{\cA} \mid\mid P ) \geq \min_{Q \in \cA} \Div (Q \mid\mid P)$ gives part (1b).

For part (2a), note that 
\begin{align*}
\Div ( \mu_{\cA} \mid\mid \omega_{\cA}^n ) 
&= \Xent ( \mu_{\cA} , \omega_{\cA}^n ) - \Xent ( \mu_{\cA} ) 
= \evp{- \ln \lrp{ \prod_{i=1}^{n} \omega_{\cA} (Z_i) } }{Z \sim \mu_{\cA}} - \Xent ( \mu_{\cA} ) \\
&= \sum_{i=1}^{n} \evp{- \ln \lrp{ \omega_{\cA} (Z_i ) } }{Z \sim \mu_{\cA}} - \Xent ( \mu_{\cA} ) 
= \sum_{i=1}^{n} \evp{- \ln \lrp{ \omega_{\cA} (u) } }{u \sim \omega_{\cA}} - \Xent ( \mu_{\cA} ) \\
&= n \Xent ( \omega_{\cA} ) - \Xent ( \mu_{\cA} ) 
\end{align*}
Using this on Eq. \eqref{eq:sanovexact}, 
$\frac{1}{n} \ln \prp{ \hat{P}_{n} \in \cA}{} = - (\Div (\omega_{\cA} \mid\mid P ) + \Xent ( \omega_{\cA} ) ) + \frac{1}{n} \Xent ( \mu_{\cA} ) 
= - \Xent (\omega_{\cA} , P ) + \frac{1}{n} \Xent ( \mu_{\cA} )
\geq - \Xent (\omega_{\cA} , P )$. 
This proves part (2a). 
Using the convexity of $\cA$ and $\Xent (\omega_{\cA} , P ) \leq \max_{Q \in \cA} \Xent (Q , P )$ gives part (2b).

\end{proof}




\bibliography{sample}

\end{document}